\documentclass{article}

\usepackage[preprint]{neurips_2023}

\usepackage[utf8]{inputenc} 
\usepackage[T1]{fontenc}    
\usepackage{hyperref}       
\usepackage{url}            
\usepackage{booktabs}       
\usepackage{amsfonts}       
\usepackage{nicefrac}       
\usepackage{microtype}      
\usepackage[table,xcdraw]{xcolor}         
\usepackage{graphicx}
\usepackage{subcaption} 
\usepackage{relsize}
\usepackage{tablefootnote}
\usepackage{times}
\usepackage{epsfig}
\usepackage{graphicx}
\usepackage{amsmath}
\usepackage{amssymb}
\usepackage{multirow}
\usepackage{pifont}
\usepackage{booktabs}
\usepackage{multirow}
\usepackage{bbm}

\usepackage{amsthm}

\newtheorem{proposition}{Proposition}

\title{Centered Self-Attention Layers}

\author{%
  Ameen Ali \\
  Department of Computer Science\\
  Tel Aviv University\\
  \texttt{ameenali@mail.tau.ac.il} \\
  \And
  Tomer Galanti \\
  Department of Brain and Cognitive Sciences \\
  Massachusetts Institute of Technology \\
  \texttt{galanti@mit.edu} \\
  \AND
  Lior Wolf \\
  Department of Computer Science \\
  Tel Aviv University\\
  \texttt{wolf@mail.tau.ac.il} \\
}

\begin{document}

\maketitle

\begin{abstract}
The self-attention mechanism in transformers and the message-passing mechanism in graph neural networks are repeatedly applied within deep learning architectures. We show that this application inevitably leads to oversmoothing, i.e., to similar representations at the deeper layers for different tokens in transformers and different nodes in graph neural networks. Based on our analysis, we present a correction term to the aggregating operator of these mechanisms. Empirically, this simple term eliminates much of the oversmoothing problem in visual transformers, obtaining performance in weakly supervised segmentation that surpasses elaborate baseline methods that introduce multiple auxiliary networks and training phrases. In graph neural networks, the correction term enables the training of very deep architectures more effectively than many recent solutions to the same problem.
\end{abstract}

\section{Introduction}

Training very deep neural networks is known to suffer from the problem of vanishing and exploding gradients: during the backward pass, as the gradient propagates from the top layers to the input, its magnitude is multiplied from one layer to the previous one, causing it to explode in the case of a factor that is larger than one, or vanish in the case of a factor that is less than one.  Here, we explore a related phenomenon called oversmoothing \cite{wang2022anti,shi2022revisiting,wang2022scaled}, which exists in the forward pass of deep networks that employ attention, such as Transformers~\cite{NIPS2017_3f5ee243}. As layers are added, the representation of the tokens becomes increasingly similar, and their diversity and expressivity are lost. 

Oversmoothing is often ignored, since for image-level tasks, such as image classification, the top layers represent high-level patterns at coarser resolutions. However, when the spatial maps matter, for example, in weakly supervised semantic segmentation (WSSS), the problem is well-known \cite{ru2023token}. The state-of-the-art solutions for this problem in WSSS are extremely elaborate and involve multiple components, loss terms, and pretext tasks \cite{zhang2021weakly,shimoda2019self}. In other fields, such as image segmentation, transformers are trained with stochastic depth, which may help mitigate oversmoothing during training~\cite{strudel2021segmenter}. Oversmoothing in transformers was also investigated based on the Fourier spectrum~\cite{wang2022anti} or on graph theory~\cite{shi2022revisiting}. 

A similar issue exists in graph neural networks (GNNs), where the messages have a similar additive form. As a result, it is well-known that increasing the depth of GNNs is detrimental to their performance, and therefore, practitioners typically favor training shallow GNNs~\cite{cai2020note,wu2022non,keriven2022not}. The GNN literature has proposed multiple solutions, focusing on normalization~\cite{ba2016layer,zhao2019pairnorm,guo2023contranorm,dasoulas2021lipschitz}, skip connections~\cite{xu2021optimization,xu2020neural}, and attention regularization~\cite{dasoulas2021lipschitz}. 

In this work, we explore a novel approach to address the oversmoothing problem by studying the impact of softmax activations. Our analysis focuses on the oversmoothing phenomenon in basic multi-layered Transformer architectures, which arises from repeated multiplication between attention matrices. We argue that since the singular values of these matrices fall within the range of (-1,1), except for the leading singular value, which is 1, the network's output converges to a rank-1 matrix as we increase the depth.

To mitigate this issue, we propose a modification to the attention mechanism, by introducing a very simple centering term. This term adjusts the softmax activations to add up to 0 instead of 1, effectively countering the oversmoothing effect. Through several synthetic simulations, we provide empirical evidence that three well-established types of Transformers suffer from oversmoothing at initialization when the softmax activations are not adjusted. However, once we incorporate the centering term into the self-attention layers, the oversmoothing problem is alleviated.


Our experiments demonstrate that the centering term we propose is sufficient for achieving state-of-the-art weakly supervised semantic segmentation results, within a much simpler framework than recent contributions. In the task of training deep GNNs we outperform recently proposed sophisticated normalization techniques, simply by adding our correction term to a vanilla GNN. 

\section{Related Works}




Multiple contributions explore different strategies for organizing the self-attention, residual connection, and normalization layers within each block of the Transformer architecture. One widely used approach, known as the Post-LN (Post Layer Normalization) variant, has been employed in BERT~\cite{devlin-etal-2019-bert}, RoBERTa~\cite{liu2019roberta}, and ALBERT~\cite{lan2020albert}. This approach applies layer normalization after the output of each residual block. Conversely, the Pre-LN variant, adopted in the GPT series~\cite{NEURIPS2020_1457c0d6}, ViT~\cite{dosovitskiy2020image}, and PALM~\cite{chowdhery2022palm}, applies layer normalization before the input of each residual block. Despite the successes of these approaches, recent findings indicate that Post-LN Transformers often encounter the vanishing gradients problem~\cite{pmlr-v119-xiong20b}, while Pre-LN Transformers tend to exhibit severe oversmoothing~\cite{liu-etal-2020-understanding}.

In an effort to mitigate both vanishing gradients and oversmoothing, ResiDual~\cite{xie2023residual} combines elements from both the Post-LN and Pre-LN methods, aiming to address the challenges of vanishing gradients and oversmoothing in Transformer models. In our synthetic experiments, we demonstrate the oversmoothing phenomenon in all three types (Pre- and Post-LN as well as ResiDual) and show that our contribution can mitigate it.


Oversmoothing has been identified as a major challenge for vision transformers \cite{ru2023token,wang2022anti,gong2021improve}, particularly in dense prediction tasks, where precise localization of objects is crucial \cite{ru2023token}. 
Several recent studies have attempted to address this issue by proposing novel architectures or training strategies that explicitly tackle oversmoothing \cite{wang2022anti,shi2022revisiting}, utilizing the diversity of intermediate layers \cite{ru2023token}. These techniques aim to improve the model's ability to distinguish between different regions of the image and assign appropriate weights to different pixels or patches.

In the context of weakly supervised semantic segmentation using image-level labels, Class Activation Maps (CAMs) \cite{ru2023token,pan2022learning,ru2022weakly,zhang2020reliability,araslanov2020single} have been widely used as a first step towards generating initial pseudo-labels for training the decoder model. CAMs can localize the object of interest by highlighting the most discriminative regions for a given class. However, oversmoothing in transformers results in maps that lose spatial information. 
To address this, various improvements have been proposed in the literature, such as utilizing the semantic diversity of intermediate layers \cite{ru2023token} and incorporating multiple class tokens to produce class-specific attentions \cite{xu2022multi}. 

Oversmoothing in GNNs 
is manifested as node embeddings that become too similar after several rounds of message passing \cite{wu2022non,rusch2023survey}. 
One approach to overcome this is to incorporate skip connections \cite{xu2021optimization,xu2020neural}, which was used in 
Graph Attention Networks (GATs)~\cite{velivckovic2017graph} and Graph Convolutional Networks with attention mechanisms (GCANs)~\cite{guo-etal-2019-attention}. 
Another approach is to incorporate normalization techniques: \cite{zhao2019pairnorm} proposes PairNorm, in which pairwise node embeddings in a GNN layer are prevented from becoming too similar. Very recently \cite{guo2023contranorm} proposed ContraNorm, which applies a shattering operator to the representations in the embedding space to prevent dimensional collapse.



\section{Method}

{\bf Background and notation.\enspace} The core building block of the Transformer architecture is the self-attention mechanism, which enables the model to learn attention patterns over its input tokens. Let $X \in \mathbb{R}^{n \times d}$ denote the input feature matrix to the network, where $n$ represents the number of tokens and $d$ is the dimension of each token. The self-attention mechanism, denoted as $S:\mathbb{R}^{n\times d} \to \mathbb{R}^{n \times d}$, can be expressed as:
\begin{equation}
S(X;W) = \textnormal{softmax}\left(\frac{XW_Q(XW_K)^{\top}}{\sqrt{d}} \right) XW_V
\end{equation}
Here, $W_K \in \mathbb{R}^{d \times d}$, $W_Q \in \mathbb{R}^{d \times d}$, and $W_V \in \mathbb{R}^{d \times d}$ are the key, query, and value trainable weight matrices, respectively, and the softmax function, $\textnormal{softmax}(\cdot )$ applies the softmax function $(x_1,\dots,x_d) \mapsto \left(\frac{\exp(x_1)}{\sum^{d}_{i=1}\exp(x_i)},\dots,\frac{\exp(x_d)}{\sum^{d}_{i=1}\exp(x_i)} \right)$ row-wise. The scaling factor, $1/\sqrt{d}$, is used to prevent large values in the dot products from causing the softmax function to have extremely small gradients. For brevity, below we will employ $K = XW_K$, $Q = XW_Q$, and $V = XW_V$. 

The self-attention mechanism allows the model to weigh the importance of each token in the input sequence relative to the others, enabling the model to better capture long-range contextual information.

The Transformer architecture consists of multiple layers, each containing a multi-head self-attention mechanism followed by a position-wise feed-forward network. These layers are combined with residual connections and layer normalization to create a deep model based on attention.


Graph Neural Networks (GNNs) process graphs of the form $G=(V,E)$, where, $V$ is a set of nodes and $E$ a set of edges between pairs of nodes. GNNs operate on graphs by performing a core updating mechanism called message-passing, where each node sends messages to its neighbors, which are then used to update the node's feature vector.

The message-passing procedure in GNNs involves two steps: aggregation and combination. In the aggregation step, the feature vectors of neighboring nodes are combined into a single message for each node. This is typically done by computing a weighted sum of the neighboring nodes' feature vectors, where the weights are determined by a learnable function, such as a neural network or a graph convolutional filter. In the combination step, the message is then combined with the node's previous feature vector to produce a new feature vector.

It has been demonstrated by \cite{guo2023contranorm} that the self-attention matrix used in Transformers can be interpreted as a normalized adjacency matrix of a corresponding graph. Let us consider a graph $G$ with $m$ nodes and let $A\in \{0,1\}^{m\times m}$ be its adjacency matrix, which is a symmetric matrix where $A_{ij}$ specifies if $i$ and $j$ are connected by an edge.

To construct the normalized adjacency matrix of a graph $G$, we set the weight of the edge between node $i$ and node $j$ as $\exp(Q_i^\top K_j)$, where $Q = XW_Q$ and $K = XW_K$ such that $X$ is the input of the given layer and $W_Q,W_K \in \mathbb{R}^{d \times d}$ are trainable weight matrices. Then, the $(i,j)$ entry of the normalized adjacency matrix $\tilde{A}$ is given by $\tilde{A}_{i,j} = \frac{A_{i,j}}{D_{i,i}}$, where $D$ is a diagonal matrix with elements $D_{ii}=\sum^{d}_{j=1} A_{i,j}$ equal to the sum of the $i$-th row of $A$. In particular, the normalized adjacency matrix $\tilde{A}$ can be expressed as follows $\tilde{A}_{ij} = \frac{\exp(Q^\top_iK_j)}{\sum^{d}_{k=1} \exp(Q^\top_iK_k)}$. This normalized adjacency matrix $A$ has the same form as the self-attention matrix used in Transformers, and it plays a crucial role in the message-passing mechanism by determining which nodes exchange information with one another.

\subsection{Understanding oversmoothing}

In Transformers, the attention mechanism computes a set of attention coefficients that are used to weight the contribution of each token in the input sequence to the output representation. These attention weights are computed based on a similarity score between the query and each of the key vectors associated with the input tokens.

However, as the depth of the transformer network increases, the product over the attention maps tends to a low-rank matrix, resulting in highly degenerate tokens that are collinear with each other. This is because the attention mechanism gives an attention matrix whose largest singular value is 1 and the other singular values are members of $(0,1)$.

\begin{proposition}\label{prop:1}
Let $X \in \mathbb{R}^{n \times d}$ be an input and $f_L(X) = S \circ \dots \circ S(X)$ be a Transformer, where each layer $S(X) = S(X;W)$ corresponds to the same weight matrices $W_K$, $W_Q$, and $W^l_V = I$. Then, $f_L(X)$ converges to some matrix $f^*(X)$ of rank $\leq 1$.
\end{proposition}

\begin{proof}
We can write $f(X) = S(X;W)^k X$, and since $S(X;W)$ is a matrix whose rows are non-negative and sum to 1, its largest eigenvalue is 1 and all of the other eigenvalues lie in $(-1,1)$. Thus, we can write $S(X;W) = PDP^{-1}$, where $D$ is a diagonal matrix whose diagonal is the eigenvalues of $S(X;W)$ and $P$ is an orthogonal matrix. Using this, we can write $S(X;W)^k = PD^kP^{-1}$, where $D^k$ converges to a matrix whose diagonal is $(1,0,\dots,0)$ and $S(X;W)^k$ converges to a matrix of rank $\leq 1$. Thus, $f(X)$ converges to a matrix $f^*(X)$ of rank $\leq 1$.
\end{proof}

\begin{proposition}\label{prop:2}
Let $X \in \mathbb{R}^{n \times d}$ be an input and $f_L(X) = S_L \circ \dots \circ S_1(X)$ be a Transformer, where each layer $S_l(X) = S(X;W^l)$ corresponds to randomly initialized weight matrices $W^l_K$, $W^l_Q$, and $W^l_V=I$. Assume that $f_L(X)$ converges almost surely to some random variable $f^*(x)$ with respect to the selection of the weights as a function of $L \to \infty$, and assume $f^*(X) \perp S(f^*(X);W)\cdot f^*(X)$ for $W$ independent of $f^*(X)$. Then, $\textnormal{rank}(\mathbb{E}[f^*(X)])\leq 1$.
\end{proposition}

\begin{proof} 
Due to the convergence of $f_L(X)$ to the random variable $V_1 = f^*(X)$, we can write $V_1 \sim S(V_1;W^1) \cdot V_1$ with respect to the randomness of $f^*(X)$ and the parameters $W$. In particular, we can write $V_1 \sim V_k$ where $V_{k} =S(V_{k-1};W^{k-1}) \cdot V_{k-1}$ with $W^1,\dots,W^k$ random i.i.d. weights. In particular, $\mathbb{E}[V_1] = \mathbb{E}[V_k] = \mathbb{E}[S(V_{k-1};W^{k-1}) \cdot \ldots \cdot S(V_{1};W^{1}) \cdot V_1]$. Since $f^*(X) \perp S(f^*(X);W)\cdot f^*(X)$, we obtain that $V_{i+1} \perp V_{i}$ for all $i \in [k]$ and in particular, $S(V_{i};W^i) \perp S(V_{i-1};W^{i-1})$. Since the sequence $V_1,\dots,V_{k}$ is Markovian, we have: 
\begin{equation}
\begin{aligned}
\mathbb{E}[V_k] = \mathbb{E}[V_1] &= \mathbb{E}[S(V_{k-1};W^{k-1})]\cdot \ldots \cdot \mathbb{E}[S(V_{2};W^{2})] \cdot \mathbb{E}[S(V_{1};W^{1}) \cdot V_1] \\ 
&=   \mathbb{E}[S(V_{k-1};W^{k-1})]\cdot \ldots \cdot \mathbb{E}[S(V_{2};W^{2})] \cdot \mathbb{E}[V_2] \\
&= \mathbb{E}[S(V_{1};W^{1})]^{k-2} \mathbb{E}[V_1]
\end{aligned}
\end{equation} 
We note that for any $W^1$ and $V_1$, $S(V_{1};W^{1})$ is a matrix whose rows are non-negative and sum to 1, its largest eigenvalue is 1 and all of the other eigenvalues lie in $(-1,1)$. Similarly to the proof of Prop.~\ref{prop:1} we obtain that $\mathbb{E}[S(V_{1};W^{1})]^{k-2}$ converges to a matrix of rank $\leq 1$, and therefore, $\mathbb{E}[S(V_{1};W^{1})]^{k-2} \mathbb{E}[V_1]$ also converges to a matrix of rank $\leq 1$. We note that this sequence is actually constant and equal to $\mathbb{E}[V_1] = \mathbb{E}[f^*(X)]$ and therefore, $\mathbb{E}[f^*(X)]$ is a matrix of rank $\leq 1$.
\end{proof}

When the attention weights become uniform or nearly uniform, the model may lose the ability to distinguish between important and unimportant features in the image, resulting in a loss of detail and specificity in the image representation. This can lead to less accurate predictions and less informative visualizations, such as CAM maps.

Therefore, addressing the oversmoothing problem in transformers involves preventing the attention weights from converging toward uniformity. We address the issue of oversmoothing in Transformers by forcing the attention weights to sum to 0 instead of 1.

\begin{figure}
    \centering
    \begin{tabular}{c@{}ccccc}
    &\textbf{Pre LN} & \textbf{Post LN} & \textbf{ResiDual}
    \\
    \rotatebox{90}{{\bf ~~~~~~~~~~~~~Uniform}} & \includegraphics[width=0.3000000\textwidth]{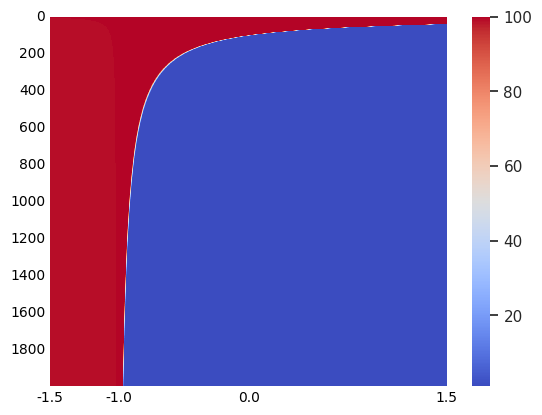}
    &
    \includegraphics[width=0.3000000\textwidth]{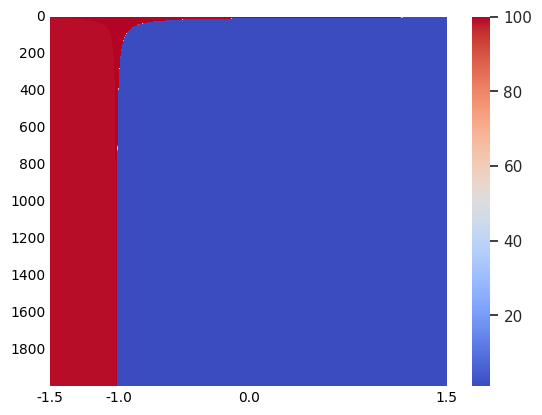}
    &
    \includegraphics[width=0.3000000\textwidth]{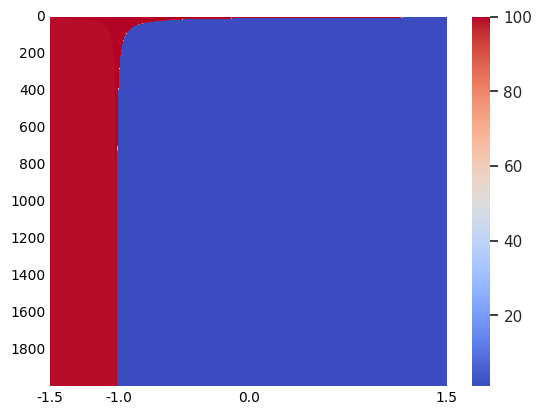}
    \\
    \rotatebox{90}{{\bf ~~~~~~~~~~~~~Identity}} & \includegraphics[width=0.3000000\textwidth]{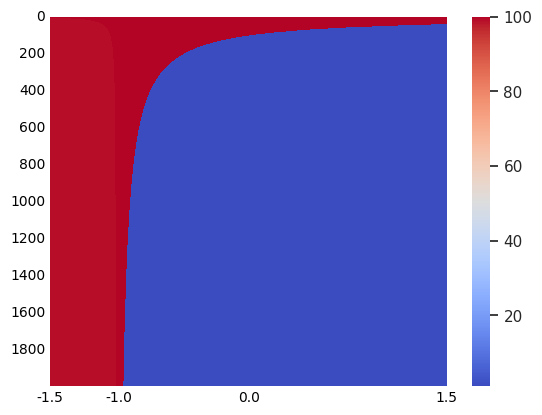}
    &
    \includegraphics[width=0.3000000\textwidth]{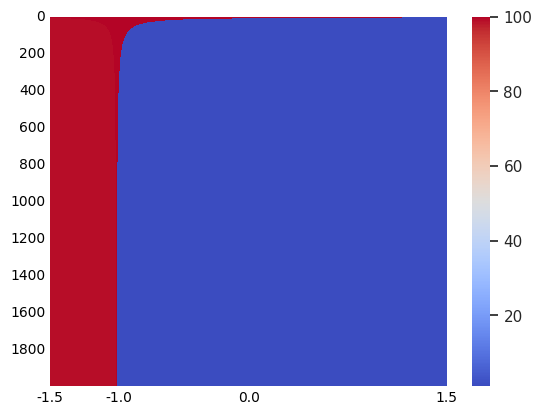}
    &
    \includegraphics[width=0.3000000\textwidth]{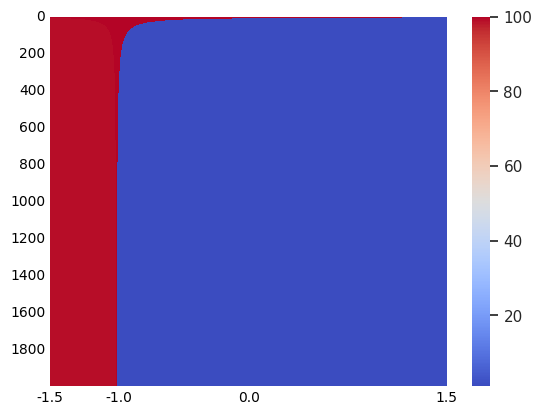}
    \\
    \rotatebox{90}{{\bf ~~~~~~~~~~~~~Normal}} & \includegraphics[width=0.3000000\textwidth]{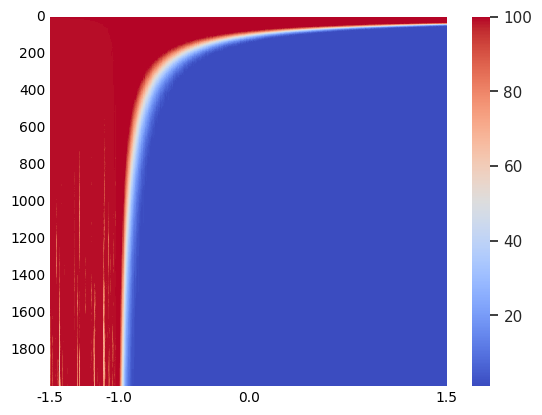}
    &
    \includegraphics[width=0.3000000\textwidth]{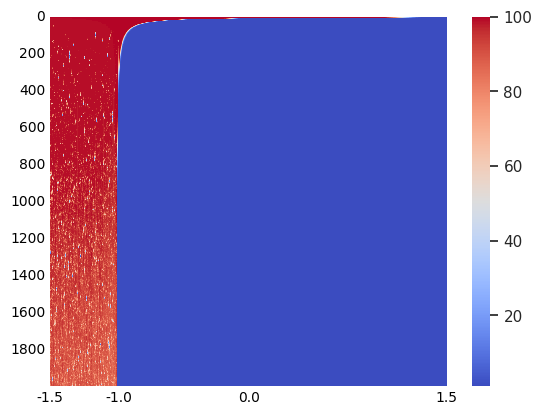}
    &
    \includegraphics[width=0.3000000\textwidth]{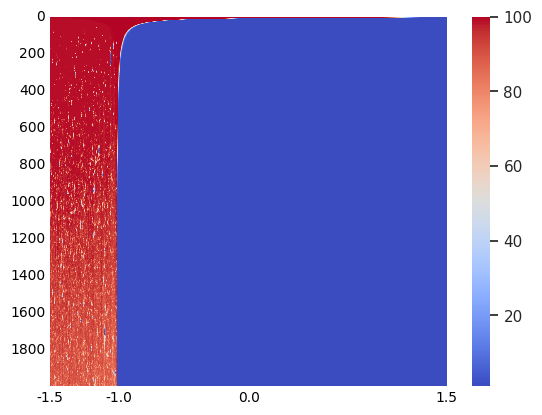}
    \\
\end{tabular}
\caption{{\bf Oversmoothing in Transformers.} The graph illustrates the relationship between the rank of the output and two parameters: the offset parameter $\gamma$ (x-axis) and the depth of the network (y-axis). Each graph represents a specific type of Transformer architecture with different block types (Pre-LN, Post-LN, or ResiDual) and weight initialization. It is observed that the rank approaches 1 as the depth increases to infinity exclusively when $\gamma > -1$.}
\label{fig:simulation}
\end{figure}

\subsection{Centered Self-Attention Layers}

In the previous section, we showed that at initialization, deep Transformers suffer from oversmoothing because all of the attention matrices' eigenvalues lie in $(-1,1)$ except for the largest eigenvalue, which is $1$. In order to address the oversmoothing problem, we can offset the sum of each row by some constant $\alpha \in \mathbb{R}$ to ensure that the eigenvalues of the matrix lie in a larger interval. One way to ensure this is to offset the attention rows to sum to $1+\gamma$ instead of 1, by the following modification: 
\begin{equation*}
\begin{aligned}
S_{\gamma}(X;W) =  \left(\textnormal{softmax}\left(\frac{QK^{\top}}{\sqrt{d}}\right) + \gamma\frac{\mathbbm{1} \mathbbm{1}^{\top}}{n}\right) V = \textnormal{softmax}\left(\frac{XW_Q(XW_K)^{\top}}{\sqrt{d}} + \gamma\frac{\mathbbm{1} \mathbbm{1}^{\top}}{n} \right) XW_V 
\end{aligned}
\end{equation*}
Specifically, we focus on the case where $\gamma=-1$, in which the softmax activations are adjusted to zero. We note that potentially one could obtain the same effect with different values for the offset. In order to better understand how the offset affects the degree of oversmoothing, we conducted several simulations in Section~\ref{sec:simulations}. 



\subsection{Application to Weakly Supervised Semantic Segmentation}\label{sec:wsss}

The ViT model \cite{dosovitskiy2020image} involves the initial splitting of an image into patches to generate patch tokens. Subsequently, these tokens are concatenated with a learnable class token and input into the Transformer encoder, resulting in the final patch and class tokens. The multi-head self-attention is a crucial aspect of the Transformer block, facilitating global feature interaction.

The oversmoothing problem has been identified to adversely affect the output patch tokens, rendering them overly uniform after multiple Transformer blocks \cite{wang2022anti,wang2022scaled,ru2023token}. This issue severely impacts the CAMs of Transformer. Notwithstanding, the early layers of the model can still preserve semantic diversity. Consequently, following previous work \cite{ru2023token} we have introduced an auxiliary classifier in an intermediate layer of the transformer model, alongside the application of the proposed correction term for the self-attention layers, to counteract the oversmoothing issue.
The resulting loss function can be written as 
$\mathcal{L} = \mathcal{L}_{cls} + \mathcal{L}^{m}_{cls}$, 
where $\mathcal{L}_{cls}$ and $\mathcal{L}^{m}_{cls}$ are Multilabel soft-margin losses of the last layer and the intermediate layer classifiers. 

In order to achieve single-stage weakly supervised semantic segmentation, we extract the class activation maps from the last classification layer. These maps are later refined using a pixel-adaptive refinement module (PAR) \cite{ru2022learning} to generate final refined pseudo labels, using the concept of pixel-adaptive convolution \cite{araslanov2020single,su2019pixel} to improve the accuracy of pseudo labels by efficiently incorporating RGB and position information of nearby pixels. 

The refined pseudo labels are then utilized to supervise the segmentation decoder, which is trained via a Cross-Entropy loss  ($\mathcal{L}_{seg}$). Additionally, a regularization loss  term~\cite{tang2018regularized} ($\mathcal{L}_{reg}$) is used, following prior research on single-stage weakly supervised semantic segmentation~\cite{ru2023token,pan2022learning,ru2022weakly}. The regularization loss term serves to enforce spatial consistency in the predicted segmentation masks, reducing the risk of segmented regions with irregular shapes or disconnected segments. 

The details of the training pipeline are given in the supplementary appendix. At test time, in order to further improve the segmentation performance, the proposed method employs multi-scale testing and dense CRF processing, which are standard techniques in semantic segmentation. While this description may seem a little elaborate, the method is entirely composed of conventional components that are used throughout the relevant literature. As our ablation study shows, without using our term, this pipeline by itself is not very effective.

\section{Experiments}

\begin{table}[t]
  \small
  \centering
    \caption{\textbf{VOC Semantic Segmentation Results}. $Sup.$ denotes the supervision type. $\mathcal{I}$: Image-level labels; $\mathcal{S}$: Saliency maps. $Net.$ denotes the backbone network (for single-stage methods) and the semantic segmentation network (for multi-stage methods). $\dagger$ denotes using ImageNet-21k~\cite{ridnik2021imagenet21k} pretrained parameters.}
  \setlength{\tabcolsep}{0.8mm}
  \begin{tabular}{lcccc|lcccc}
    \toprule
    Method & {$Sup.$}   & {$Net.$  }       & {val}                      & {test}  &     Method & {$Sup.$}   & {$Net.$  }       & {val}                      & {test}                                                                             \\ \midrule
    \multicolumn{5}{l|}{\textbf{\textit{Multi-stage methods}}.}   &            \multicolumn{4}{l}{\textbf{\textit{Multi-stage methods (continued)}}.}                                                                           \\
    RIB \cite{lee2021reducing}     & $\mathcal{I}+\mathcal{S}$ & DL-V2                     & 70.2                              & 70.0                                                        &    MCTformer \cite{xu2022multi}      & $\mathcal{I}$             & WR38                      & 71.9                              & 71.6                                                            \\
    EPS \cite{lee2021railroad}       & $\mathcal{I}+\mathcal{S}$ & DL-V2                     & 71.0                              & 71.8                                                           &     ESOL \cite{li2022expansion}     & $\mathcal{I}$             & DL-V2                     & 69.9                              & 69.3                                                                                       \\
    L2G \cite{jiang2022l2g}          & $\mathcal{I}+\mathcal{S}$ & DL-V2                     & 72.1                              & 71.7                                                      &   \multicolumn{5}{l}{\textbf{\textit{Single-stage methods}}.}                                                         \\
    RCA   \cite{zhou2022regional}   & $\mathcal{I}+\mathcal{S}$ & DL-V2                     & 72.2                              & 72.8                                                                                               & RRM \cite{zhang2020reliability}  & $\mathcal{I}$             & WR38                      & 62.6                              & 62.9                                                                                               \\
    Du  \cite{du2022weakly}    & $\mathcal{I}+\mathcal{S}$ & DL-V2                     & {\bf 72.6}                              & {\bf 73.6}                                                                                   &   1Stage \cite{araslanov2020single}& $\mathcal{I}$             & WR38                      & 62.7                              & 64.3                                                                                                \\
    RIB \cite{lee2021reducing}    & $\mathcal{I}$             & DL-V2                     & 68.3                              & 68.6                                                                               &   AFA \cite{ru2022weakly}         & $\mathcal{I}$             & MiT-B1                    & 66.0                              & 66.3                                                                                             \\
   
    ReCAM \cite{ru2023token}        & $\mathcal{I}$             & DL-V2                     & 68.4                              & 68.2                                                                          &                       SLRNet \cite{pan2022learning}    & $\mathcal{I}$             & WR38                      & 67.2                              & 67.6                                                                                               \\
    VWL \cite{ru2022weakly}         & $\mathcal{I}$             & DL-V2                     & 69.2                              & 69.2                                                                                           &   ToCo \cite{ru2023token}   
    & $\mathcal{I}$             & ViT-B$^\dagger$           & 71.1                     & 72.2  \\                                          
    W-OoD \cite{lee2022weakly}       & $\mathcal{I}$             & WR38                      & 70.7                              & 70.1     & \textbf{Ours}                           & $\mathcal{I}$             & Deit-S$^\dagger$           & \textbf{72.5}                     & \textbf{72.9}                                                                                         \\

     \bottomrule
  \end{tabular}
  \label{tab_sem_seg}
\smallskip
  \caption{Evaluation and comparison of the semantic segmentation results in mIoU on the $val$ set. $\dagger$ denotes using ImageNet-21k \cite{ridnik2021imagenet} pretrained weights.}
  \label{tab_miou}
  \centering
  \footnotesize
  \setlength{\tabcolsep}{2pt}%
  \renewcommand{\arraystretch}{1.2}
  \resizebox{\linewidth}{!}{
  \begin{tabular}{l|ccccccccccccccccccccc|c}
    \toprule
                                          & \textbf{\rotatebox[origin=c]{70}{bkg}} & \textbf{\rotatebox[origin=c]{70}{aero}} & \textbf{\rotatebox[origin=c]{70}{bicycle}} & \textbf{\rotatebox[origin=c]{70}{bird}} & \textbf{\rotatebox[origin=c]{70}{boat}} & \textbf{\rotatebox[origin=c]{70}{bottle}} & \textbf{\rotatebox[origin=c]{70}{bus}} & \textbf{\rotatebox[origin=c]{70}{car}} & \textbf{\rotatebox[origin=c]{70}{cat}} & \textbf{\rotatebox[origin=c]{70}{chair}} & \textbf{\rotatebox[origin=c]{70}{cow}} & \textbf{\rotatebox[origin=c]{70}{table}} & \textbf{\rotatebox[origin=c]{70}{dog}} & \textbf{\rotatebox[origin=c]{70}{horse}} & \textbf{\rotatebox[origin=c]{70}{motor}} & \textbf{\rotatebox[origin=c]{70}{person}} & \textbf{\rotatebox[origin=c]{70}{plant}} & \textbf{\rotatebox[origin=c]{70}{sheep}} & \textbf{\rotatebox[origin=c]{70}{sofa}} & \textbf{\rotatebox[origin=c]{70}{train}} & \textbf{\rotatebox[origin=c]{70}{tv}} & \textbf{\rotatebox[origin=c]{70}{mIoU}} \\ \midrule
    {{1Stage} \cite{araslanov2020single}} & 88.7                                   & 70.4                                    & {35.1}                                     & 75.7                                    & 51.9                                    & 65.8                                      & 71.9                                   & 64.2                                   & 81.1                                   & 30.8                                     & 73.3                                   & 28.1                                     & 81.6                                   & 69.1                                     & 62.6                                     & 74.8                                      & 48.6                                     & 71.0                                     & 40.1                                    & 68.5                            & {64.3}                                & 62.7                                    \\
    AFA \cite{ru2022weakly}             & {89.9}                                 & {79.5}                                  & 31.2                                       & \textbf{80.7}                           & {67.2}                                  & 61.9                                      & {81.4}                                 & 65.4                                   & {82.3}                                 & 28.7                                     & {83.4}                                 & 41.6                                     & {82.2 }                                & {75.9}                                   & 70.2                                     & 69.4                                      & {53.0}                                   & {85.9}                                   & 44.1                           & 64.2                                     & 50.9                                  & {66.0}                                  \\
    {ToCo~\cite{ru2023token}}                                & 89.9                                   & \textbf{81.8}                           & 35.4                                       & 68.1                                    & 62.0                           & 76.6                                      & 83.6                                   & 80.4                                   & 87.7                                   & 24.5                                     & \textbf{88.1}                          & 54.9                                     & 87.0                          & 84.0                            & 76.0                                     & 68.2                                      & 65.6                            & 85.8                                     & 42.4                                    & 57.7                                     & 65.6                        & 69.8                                    \\
    ToCo$^\dagger$~\cite{ru2023token}                        & \textbf{91.1}                          & 80.6                                    & 48.7                              & 68.6                                    & 45.4                                    & \textbf{79.6}                             & \textbf{87.4}                          & \textbf{83.3}                          & 89.9                          & 35.8                            & 84.7                                   & 60.5                            & 83.7                                   & 83.2                                     & \textbf{76.8}                            & \textbf{83.0}                             & 56.6                                     & 87.9                            & 43.5                                    & 60.5                                     & 63.1                                  & 71.1                           
    \\ \midrule
    Ours$^\dagger$ & 89.3 & 81.1 & \textbf{61.2} & 78.6 & \textbf{73.0}& 75.4&83.1&71.0&\textbf{90.4}&\textbf{55.2}&84.7&\textbf{76.9}&\textbf{87.9}&\textbf{85.8}&72.5&70.5&\textbf{65.9}&\textbf{88.5}&\textbf{60.4}&\textbf{68.9}&\textbf{69.9}&\textbf{72.5}
    \\ \bottomrule
  \end{tabular}
  }
  \label{table:classes_iou}
\smallskip
    \caption{\textbf{Evaluation of pseudo labels}. $\dagger$ denotes using ImageNet-21k \cite{ridnik2021imagenet} pretrained parameters.}
    \label{tab_pseudo_label_s}%
  \centering
    \begin{tabular}{lcccccccc}
      \toprule
      {Method}  & 1Stage & AA\&LR & SLRNet & 
      AFA & 
      ViT-PCM & ViT-PCM + CRF &  ToCo & {\bf Ours} \\ 
      
      Backbone  & WR38 & WR38 & WR38 & MiT-B1 & ViT-B$^\dagger$ & ViT-B$^\dagger$ & ViT-B$^\dagger$ & Deit-S$^\dagger$ \\
      mIoU  & 66.9 & 68.2 & 67.1 & 68.7 & 67.7 & 71.4 & 73.6 & {\bf 74.1} \\
      \bottomrule
    \end{tabular}
  
\smallskip
  \centering
  \caption{\textbf{Ablation Study.} $\textbf{M}$: CAM from the final layer; $\textbf{M}^{m}$: auxiliary CAM from the intermediate layer; $Seg.$: semantic segmentation results.}
  \small
  \setlength{\tabcolsep}{0.008\textwidth}
  {
    \begin{tabular}{lccccccccc}
      \toprule
      &&& \multicolumn{3}{c}{w/o centering term} & \multicolumn{3}{c}{with centering term}\\
      \cmidrule(r){4-6}
      \cmidrule(l){7-9}
      Layers                                                       &  $\mathcal{L}_{seg}$ & $\mathcal{L}_{reg}$ & $\textbf{M}$  & $\textbf{M}^{m}$ & $Seg.$  &   
      $\textbf{M}$  & $\textbf{M}^{m}$ & $Seg.$\\\midrule
      $\mathcal{L}_{cls}$              &                          &                     &          28.1& --               & --            &    56.1      & --             & --            \\

      $\mathcal{L}_{cls} + \mathcal{L}^{m}_{cls}$                                   \     &                     &                     &     26.3      &      57.8        & --              &      57.2     &   64.6           & --            \\
    

      $\mathcal{L}_{cls} + \mathcal{L}^{m}_{cls}$  & \checkmark            &                                       &      66.4     &      67.8        &      66.7 &70.1&65.3&70.4     \\
     
      $\mathcal{L}_{cls} + \mathcal{L}^{m}_{cls}$  & \checkmark &                      \checkmark                   & 68.0 & 66.3 & 68.9&   71.4        &     63.5         &     71.8   \\
      \bottomrule
    \end{tabular}
  }
  \label{tab_ablation}%

\end{table}

\begin{table}[t]
	\centering
	\caption{Test accuracies (\%) for GNNs with 2, 4, 8, 16, and 32 layers. For each layer, the best accuracy is marked in bold, and the second best is underlined. Results are averaged over 5 runs.  We used the default
hyperparameters of the ContraNorm method \cite{guo2023contranorm} without additional optimizations}
	\vspace{0.1 in}
	\resizebox{\textwidth}{!}{
	\begin{tabular}{ll cc ccc}
		\toprule
		Dataset & Model & $\#$L=2 & $\#$L=4 & $\#$L=8 & $\#$L=16 & $\#$L=32 \\
		\midrule
		\multirow{4}{*}{Cora} & Vanilla GCN & \textbf{81.75 $\pm$ 0.51} & 72.61 $\pm$ 2.42 & 17.71 $\pm$ 6.89 & 20.71 $\pm$ 8.54 & 19.69 $\pm$ 9.54 \\
		~ & w/ LayerNorm \cite{ba2016layer} & \underline{79.96 $\pm$ 0.73} & \textbf{77.45 $\pm$ 0.67} & 39.09 $\pm$ 4.68 & 7.79 $\pm$ 0.00 & 7.79 $\pm$ 0.00 \\
		~ & w/ PairNorm \cite{zhao2019pairnorm} & 75.32 $\pm$ 1.05 & 72.64 $\pm$ 2.67 & 71.86 $\pm$ 3.31 & 54.11 $\pm$ 9.49 & 36.62 $\pm$ 2.73 \\
		~ & w/ ContraNorm \cite{guo2023contranorm} & 79.75 $\pm$ 0.33 & 77.02 $\pm$ 0.96 & \underline{74.01 $\pm$ 0.64} & \underline{68.75 $\pm$ 2.10} & \underline{46.39 $\pm$ 2.46} \\
		~ & w/ Ours  & 79.25 $\pm$ 0.03 & \underline{77.35 $\pm$ 0.05}& \textbf{74.85 $\pm$ 0.02} & \textbf{71.15 $\pm$ 0.02} &  \textbf{52.12 $\pm$ 0.07}  \\
		\midrule
		\multirow{4}{*}{CiteSeer} & Vanilla GCN & \textbf{69.18 $\pm$ 0.34} & 55.01 $\pm$ 4.36 & 19.65 $\pm$ 0.00 & 19.65 $\pm$ 0.00 & 19.65 $\pm$ 0.00 \\
		~ & w/ LayerNorm & 63.27 $\pm$ 1.15 & \underline{60.91 $\pm$ 0.76} & 33.74 $\pm$ 6.15 & 19.65 $\pm$ 0.00 & 19.65 $\pm$ 0.00 \\
		~ & w/ PairNorm & 61.59 $\pm$ 1.35 & 53.01 $\pm$ 2.19 & 55.76 $\pm$ 4.45 & 44.21 $\pm$ 1.73 & 36.68 $\pm$ 2.55 \\
		~ & w/ ContraNorm & 64.06 $\pm$ 0.85 & 60.55 $\pm$ 0.72 &  \underline{59.30}$\pm$ 0.67 &  \underline{49.01 $\pm$ 3.49} & \underline{36.94 $\pm$ 1.70} \\
		~ & w/ Ours  & \underline{65.45 $\pm$ 0.08}  &  \textbf{61.75 $\pm$ 0.02}&  \textbf{60.08 $\pm$ 0.01 }& \textbf{52.09 $\pm$ 0.05}& \textbf{40.44 $\pm$ 0.08} \\
	    \midrule
		\multirow{4}{*}{Chameleon} & Vanilla GCN & 45.79 $\pm$ 1.20 & 37.85 $\pm$ 1.35 & 22.37 $\pm$ 0.00 & 22.37 $\pm$ 0.00 & 23.37 $\pm$ 0.00 \\
		~ & w/ LayerNorm  & 63.95 $\pm$ 1.29 & 55.79 $\pm$ 1.25 & 34.08 $\pm$ 2.62 & 22.37 $\pm$ 0.00 & 22.37 $\pm$ 0.00 \\
		~ & w/ PairNorm & 62.24 $\pm$ 1.73 &58.38 $\pm$ 1.48 & \underline{49.12 $\pm$ 2.32} & 37.54 $\pm$ 1.70 & 30.66 $\pm$ 1.58 \\
		~ & w/ ContraNorm  &\underline{64.78 $\pm$ 1.68} & \underline{58.73 $\pm$ 1.12} & 48.99 $\pm$ 1.52 &\underline{40.92 $\pm$ 1.74}& \underline{35.44 $\pm$ 3.16} \\
		~ & w/ Ours  & \textbf{66.06 $\pm$ 0.04}& \textbf{60.12 $\pm$ 0.03}&  \textbf{51.20 $\pm$ 0.05}& \textbf{46.15 $\pm$ 0.03}&  \textbf{39.05 $\pm$ 0.11} \\
		\midrule
		\multirow{4}{*}{Squirrel} & Vanilla GCN & 29.47 $\pm$ 0.96 & 19.31 $\pm$ 0.00 & 19.31 $\pm$ 0.00 & 19.31 $\pm$ 0.00 & 19.31 $\pm$ 0.00 \\
		~ & w/ LayerNorm  & 43.04 $\pm$ 1.25 & 29.64 $\pm$ 5.50 & 19.63 $\pm$ 0.45 & 19.96 $\pm$ 0.44 & 19.40 $\pm$ 0.19 \\
		~ & w/ PairNorm &43.86 $\pm$ 0.41 & 40.25 $\pm$ 0.55 & \underline{36.03 $\pm$ 1.43} & 29.55 $\pm$ 2.19 & \underline{ 29.05 $\pm$ 0.91} \\
		~ & w/ ContraNorm  &  \underline{47.24 $\pm$ 0.66} & 
  \underline{40.31 $\pm$ 0.74} & 35.85 $\pm$ 1.58 &\underline{32.37 $\pm$ 0.93} & 27.80 $\pm$ 0.72 \\
  		~ & w/ Ours  & \textbf{48.32 $\pm$ 0.03}&\textbf{43.72 $\pm$  0.07}& \textbf{38.12 $\pm$ 0.08}& \textbf{35.02 $\pm$ 0.03}& \textbf{33.10 $\pm$ 0.03}  \\

		\bottomrule    
	\end{tabular}
	}
	\label{table:gnn-acc}
 \vspace{-0.1 in}
\end{table}

\begin{table}[h]
	\centering
	\caption{Test accuracies (\%) for GNNs with 2, 4, 8, 16, and 32 layers. For each layer, the best accuracy is marked in bold, and the second best is underlined. Results are averaged over 5 runs. We optimized the hyperparameters of each one of the methods independently.}
	\vspace{0.1 in}
 \resizebox{\textwidth}{!}{
	\begin{tabular}{ll cc ccc}
		\toprule
		Dataset & Model & $\#$L=2 & $\#$L=4 & $\#$L=8 & $\#$L=16 & $\#$L=32 \\
		\midrule
		\multirow{4}{*}{Cora} & Vanilla GCN & \textbf{81.30 $\pm$  0.001} &  81.22 $\pm$ 0.022& 78.50 $\pm$ 0.020 & 72.80 $\pm$ 0.010& 20.05 $\pm$ 0.065  \\
		~ & w/ LayerNorm \cite{ba2016layer} & 80.12 $\pm$ 0.002& 79.22 $\pm$ 0.002 & 75.20 $\pm$ 0.011 & 67.95 $\pm$ 0.039 & 15.20 $\pm$ 0.092\\
		~ & w/ PairNorm \cite{zhao2019pairnorm} & 80.65 $\pm$ 0.002& 78.23 $\pm$ 0.065&  74.80 $\pm$ 0.001 & 72.10 $\pm$ 0.033 & 61.12 $\pm$ 0.032 \\
		~ & w/ ContraNorm \cite{guo2023contranorm} & 82.23 $\pm$ 0.030 & 79.75 $\pm$ 0.004& 76.45 $\pm$ 0.004& 74.35 $\pm$ 0.018& 65.44 $\pm$ 0.027  \\
		~ & w/ Ours  & 82.44 $\pm$ 0.002 & \textbf{82.02 $\pm$ 0.002} & \textbf{79.98 $\pm$ 0.054}& \textbf{75.01 $\pm$ 0.017} & \textbf{71.33 $\pm$ 0.027} \\
		\midrule
		\multirow{4}{*}{CiteSeer} & Vanilla GCN & 69.12 $\pm$ 0.001& 61.65 $\pm$ 0.004&  58.35 $\pm$ 0.011 & 30.02 $\pm$ 0.030& 17.44 $\pm$ 0.011 \\
		~ & w/ LayerNorm & 66.86 $\pm$ 0.029 & 65.12 $\pm$ 0.001 & 60.02 $\pm$ 0.015 & 20.21 $\pm$ 0.022  & 17.55 $\pm$ 0.013\\
		~ & w/ PairNorm & 67.55 $\pm$ 0.012& 65.32 $\pm$ 0.012 & 60.98 $\pm$ 0.005 & 54.88 $\pm$ 0.002& 40.65 $\pm$ 0.037 \\
		~ & w/ ContraNorm & \textbf{69.45 $\pm$ 0.002} & 64.65 $\pm$ 0.035&  58.98 $\pm$ 0.025& 54.87 $\pm$ 0.017 & 48.95 $\pm$ 0.002 \\
		~ & w/ Ours  & 69.20 $\pm$ 0.009 & \textbf{67.87 $\pm$ 0.001}& \textbf{64.65 $\pm$ 0.001} & \textbf{60.94 $\pm$ 0.054}& \textbf{57.12 $\pm$ 0.021} \\
	    \midrule
		\multirow{4}{*}{Chameleon} & Vanilla GCN & 66.23 $\pm$ 0.035 & 63.82 $\pm$ 0.015& 47.59 $\pm$ 0.050& 35.09 $\pm$ 0.002 & 22.37 $\pm$ 0.001  \\
		~ & w/ LayerNorm  & 66.03 $\pm$ 0.015 & 61.18 $\pm$ 0.020 & 49.78 $\pm$ 0.034& 29.82 $\pm$ 0.007& 22.37 $\pm$ 0.001 \\
		~ & w/ PairNorm & 67.32 $\pm$ 0.077 & 63.83 $\pm$ 0.015 & 54.17 $\pm$ 0.012 & 46.49 $\pm$ 0.023& 41.01 $\pm$ 0.015\\
		~ & w/ ContraNorm  & 66.50 $\pm$ 0.003 & 60.93 $\pm$ 0.098& 54.17 $\pm$ 0.072 & 51.54 $\pm$ 0.055&41.45 $\pm$ 0.089  \\
		~ & w/ Ours  & \textbf{66.98 $\pm$ 0.001} & \textbf{65.76 $\pm$ 0.015} & \textbf{59.98 $\pm$ 0.035} & \textbf{54.65 $\pm$ 0.005} & \textbf{50.11 $\pm$ 0.013}  \\
		\bottomrule    
	\end{tabular}
 }
	\label{table:gnn-acc}
 \vspace{-0.1 in}
\end{table}

We provide a series of experiments to validate the advantages of using a centered self-attention layer compared to standard self-attention layers, including simulations, WSSS, and GNNs. 

\subsection{Simulations}\label{sec:simulations}

In order to gain a better understanding of the impact of the importance of the offset parameter on oversmoothing, we conducted a series of experiments, using randomly initialized Transformers. These experiments study the multiple pre-existing variations of Transformer, as described in~\citep{xie2023residual} (Pre-LN, Post-LN and ResiDual). 

We denote by $R_1=X_1 \in \mathbb{R}^{n\times d}$ the input to the network, which includes $n$ tokens of dimension of $d$, and by $S_{\gamma}(X;W) = \textnormal{softmax}\left(\frac{XW_Q(XW_K)^{\top}}{\sqrt{d}} + \gamma \frac{\mathbbm{1} \mathbbm{1}^{\top}}{n} \right) XW_V$ a self-attention layer applied to $X$ with offset $\gamma \in \mathbb{R}$. Three types of Transformer architectures, $f$, are considered:  (i) The Pre-LN architecture  $f(X_1) = N(X_{L+1})$, where each layer computes the function $X_{i+1} = S_{\gamma}(N(X_i);W^i) + X_i$. (ii) The Post-LN architecture takes the form $f(X_1) = X_{L+1}$, where $X_{i+1} = N(S_{\gamma}(X_i;W^i)+X_i)$. (iii) The ResiDual architecture $f(X_1)=N(R_{L+1})+X_{L+1}$, where $X_{i+1} = N(S_{\gamma}(X_i;W^i)+X_i)$ and $R_{i+1} = R_{i} + S_{\gamma}(X_i;W^i)$. Here, $N:\mathbb{R}^{n \times d} \to \mathbb{R}^{n \times d}$ is a function that normalizes each row individually by dividing it by its norm, and $\gamma \in \mathbb{R}$ is an offset parameter.

For our experiments, we used the identity matrix $X_1 \in \mathbb{R}^{100\times 100}$ as the input to the network, representing $n=100$ tokens, each with a dimension of $d=100$. We employed rectangular weight matrices $W^i_Q, W^i_K, W^i_V \in \mathbb{R}^{100 \times 100}$. Our focus was on investigating the rank of the matrix $X_i$ generated by the network after the $i$th layer. To measure the rank of a given matrix $A$, we counted the number of singular values of the matrix $\frac{A}{\|A\|_F}$ (where $\|\cdot\|_F$ denotes the Frobenius norm) that exceeded the threshold $\epsilon=10^{-3}$.

In the first simulation, we used identity weight matrices $W^i_Q=W^i_K=W^i_V=I$, while in the second simulation, we sampled $W^i_Q$, $W^i_K$, and $W^i_V$ from a uniform distribution $U([0,1])^{100\times 100}$. Figure~\ref{fig:simulation} illustrates the rank of the output of the network as a function of the coefficient $\gamma$ (ranging from $-1.5$ to $1.5$ on the x-axis) and the depth (ranging from $1$ to $2000$ on the y-axis). Evidently, in both cases, for any coefficient $\gamma > -1$, the rank of the output matrix decreases to 1 as we progress deeper into the network. This is avoided for any $\gamma \leq -1$. Therefore, we conclude that each one of the different variations of the Transformer architecture suffers from oversmoothing that is largely resolved when our correction is used.

\begin{figure}
    \centering
    \resizebox{\linewidth}{!}{
    \begin{tabular}{@{}ccccccccc}
    \includegraphics[width=0.14\textwidth]{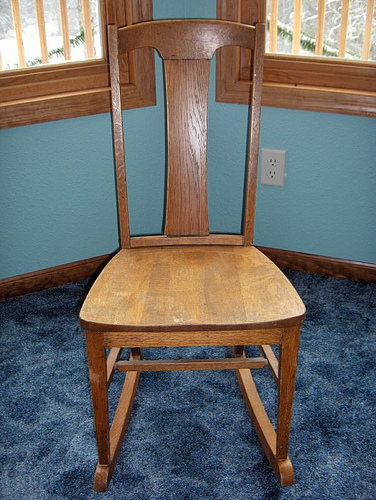}
    &
    \includegraphics[width=0.14\textwidth]{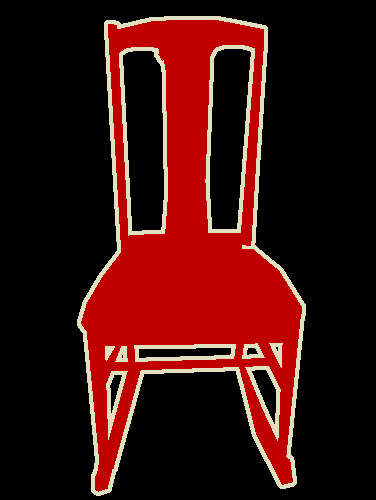}
    &
    \includegraphics[width=0.14\textwidth]{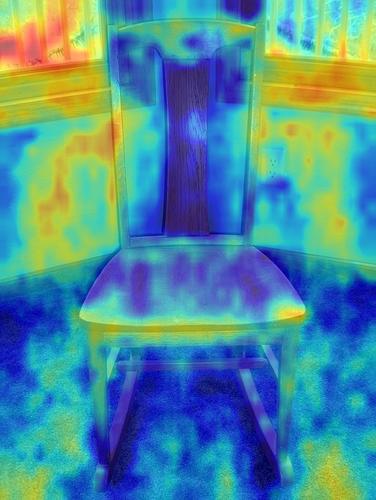}
&
    \includegraphics[width=0.14\textwidth]{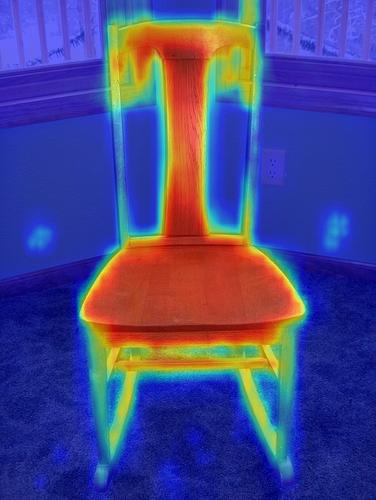}
&
    \includegraphics[width=0.14\textwidth]{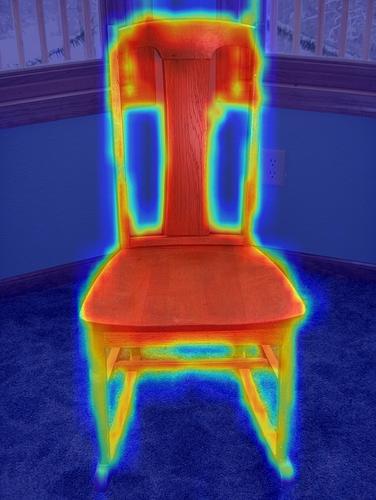}
&
    \includegraphics[width=0.14\textwidth]{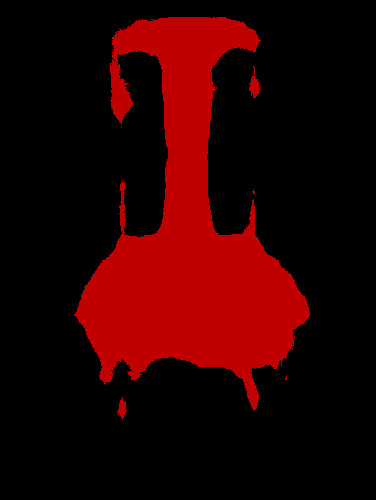}
&
    \includegraphics[width=0.14\textwidth]{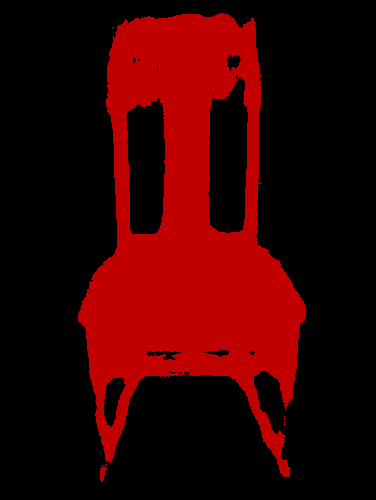}
    \\
    \includegraphics[width=0.14\textwidth]{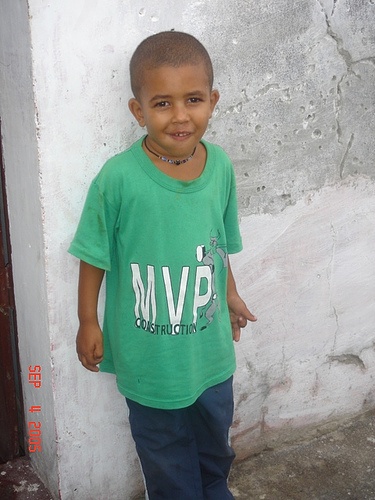}
    &
    \includegraphics[width=0.14\textwidth]{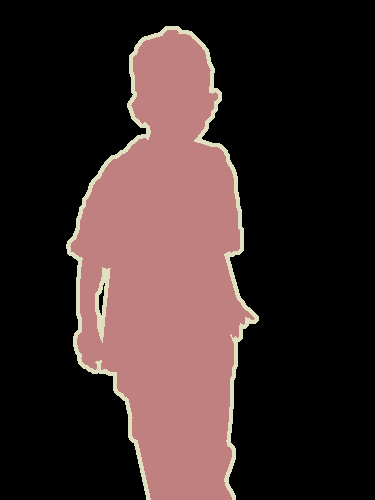}
    &
    \includegraphics[width=0.14\textwidth]{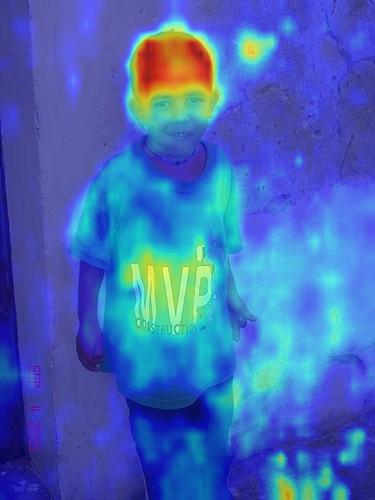}
    &
    \includegraphics[width=0.14\textwidth]{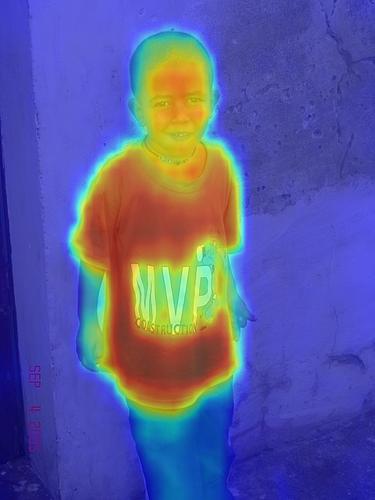}
    &
    \includegraphics[width=0.14\textwidth]{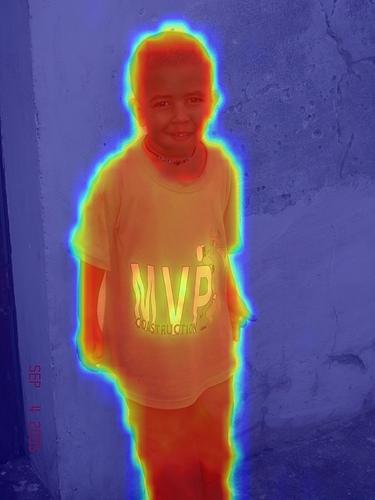}
&
    \includegraphics[width=0.14\textwidth]{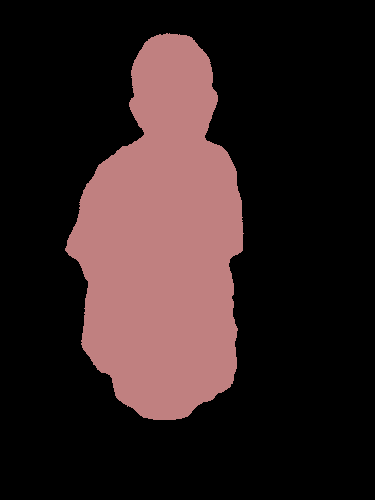}
&
    \includegraphics[width=0.14\textwidth]{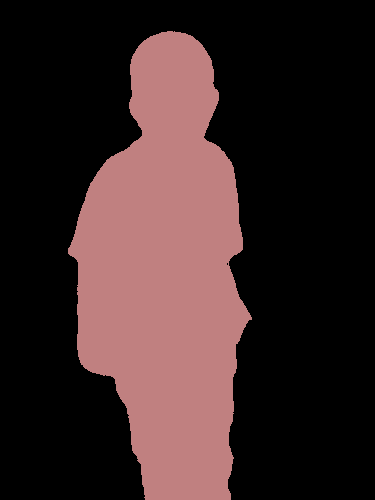}
     \\
    \includegraphics[width=0.14\textwidth]{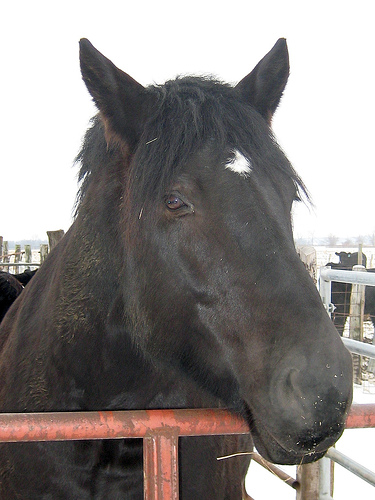}
    &
    \includegraphics[width=0.14\textwidth]{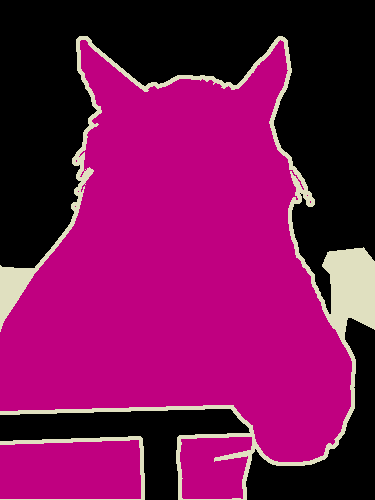}
    &
    
    \includegraphics[width=0.14\textwidth]{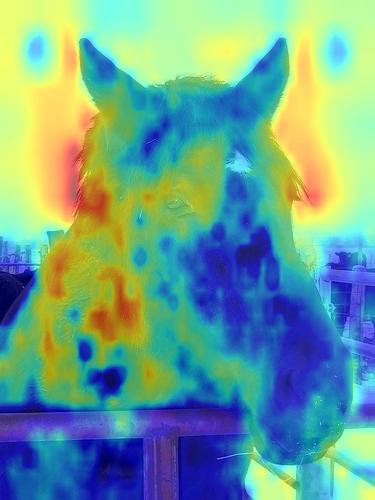}
    &
    \includegraphics[width=0.14\textwidth]{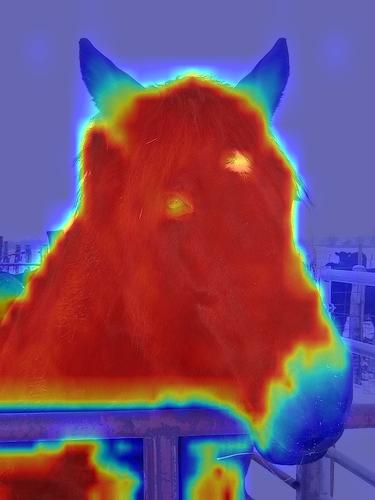}
    &
    \includegraphics[width=0.14\textwidth]{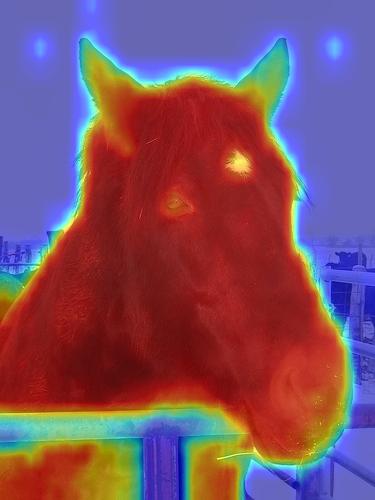}
&
    \includegraphics[width=0.14\textwidth]{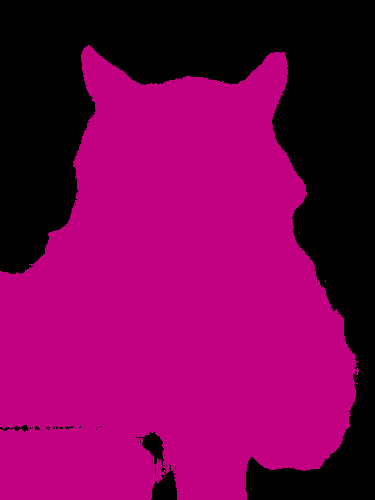}
&
    \includegraphics[width=0.14\textwidth]{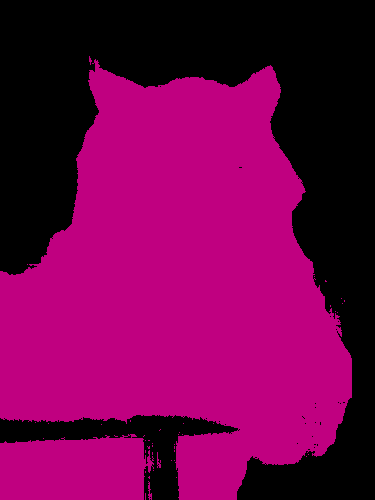}
     \\
    \includegraphics[width=0.14\textwidth]{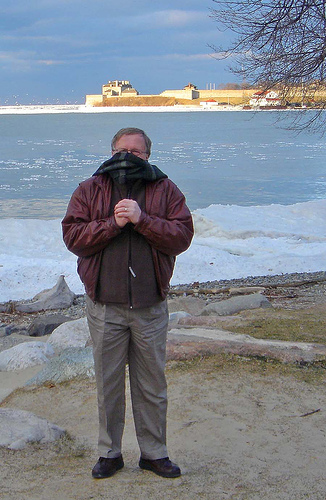}
    &
    \includegraphics[width=0.14\textwidth]{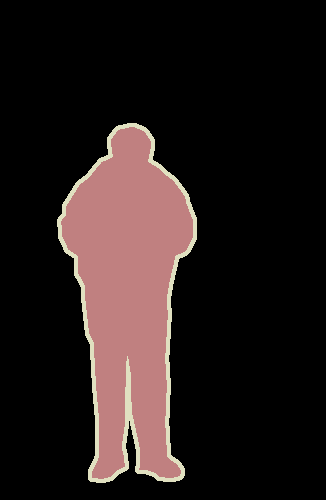}
    &
    \includegraphics[width=0.14\textwidth]{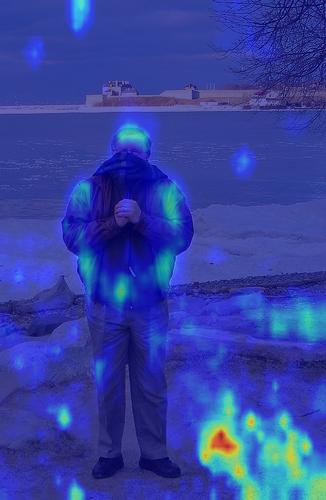}
    &
    \includegraphics[width=0.14\textwidth]{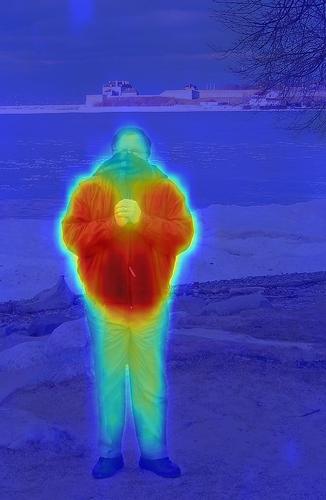}
    &
    \includegraphics[width=0.14\textwidth]{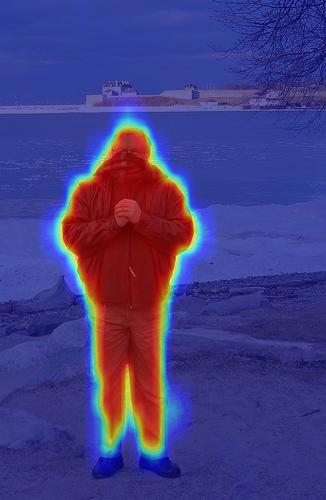}
&
    \includegraphics[width=0.14\textwidth]{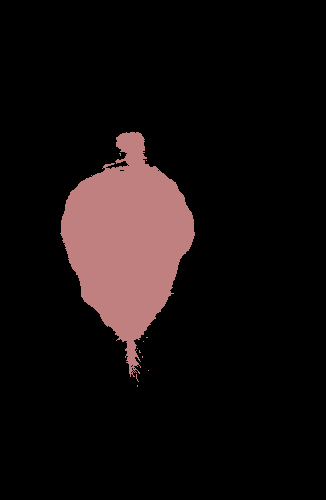}
&
    \includegraphics[width=0.14\textwidth]{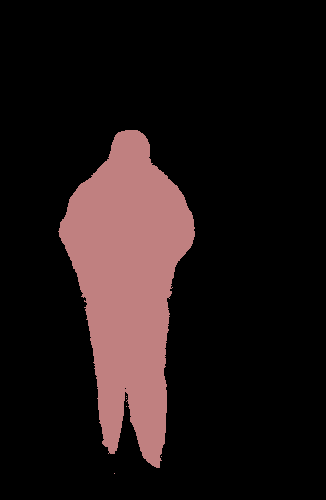}
\\
    \includegraphics[width=0.14\textwidth]{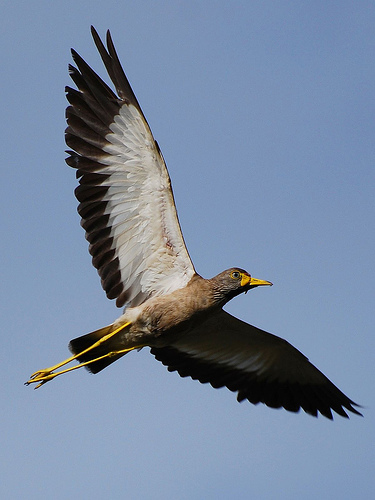}
    &
    \includegraphics[width=0.14\textwidth]{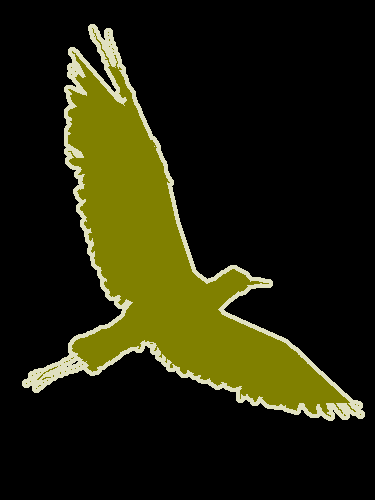}
    &
    
    \includegraphics[width=0.14\textwidth]{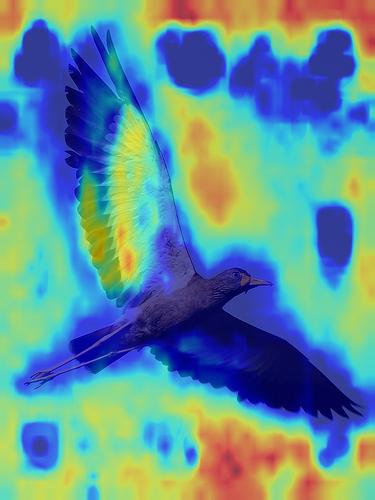}
    &
    \includegraphics[width=0.14\textwidth]{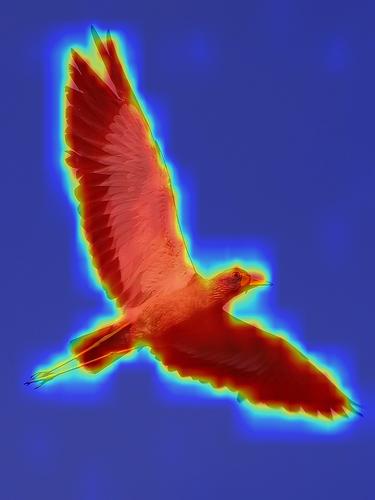}
    &
    \includegraphics[width=0.14\textwidth]{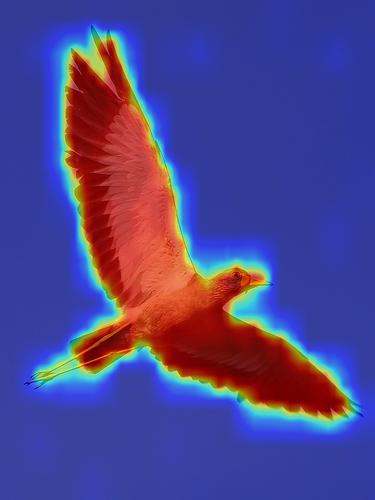}
&
    \includegraphics[width=0.14\textwidth]{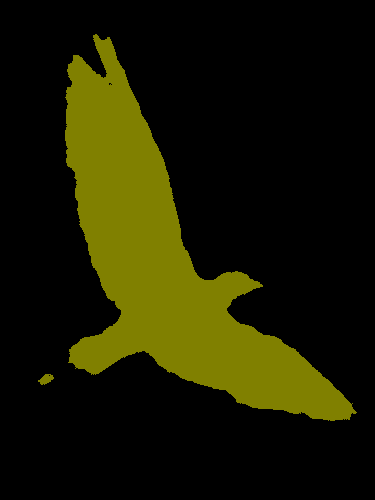}
&
    \includegraphics[width=0.14\textwidth]{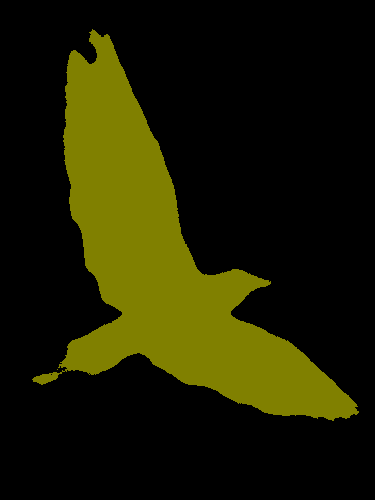}
\\
    \includegraphics[width=0.14\textwidth]{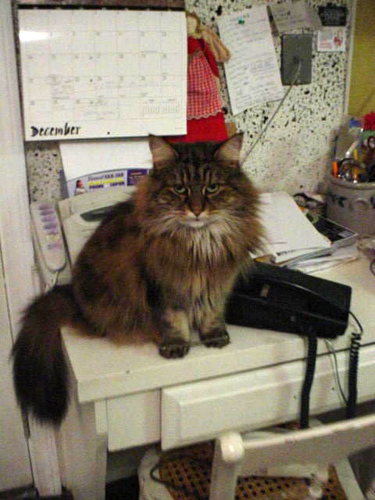}
    &
    \includegraphics[width=0.14\textwidth]{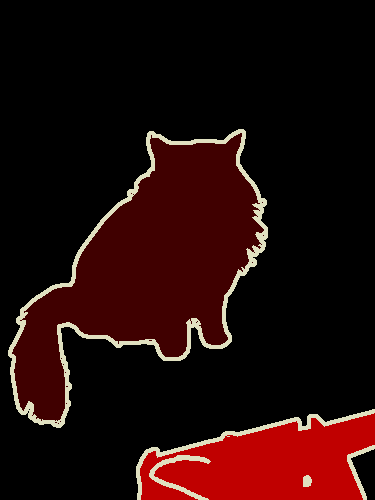}
    &
    
    \includegraphics[width=0.14\textwidth]{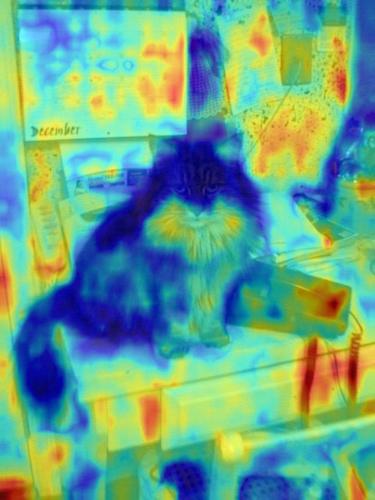}
    &
    \includegraphics[width=0.14\textwidth]{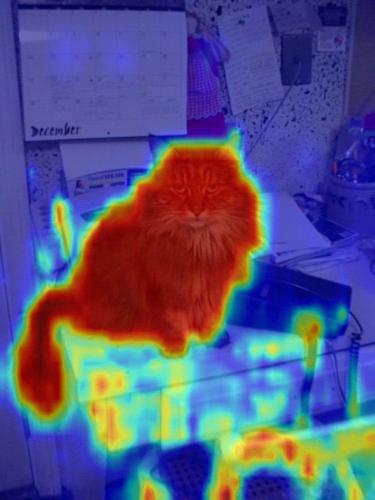}
    &
    \includegraphics[width=0.14\textwidth]{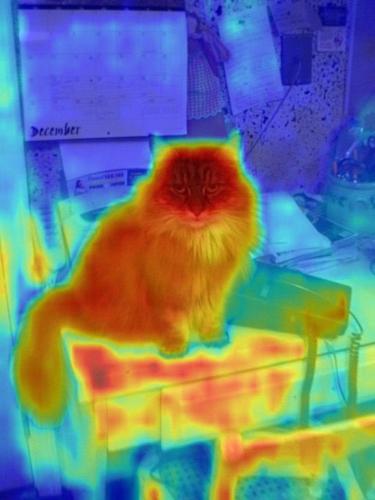}
&
    \includegraphics[width=0.14\textwidth]{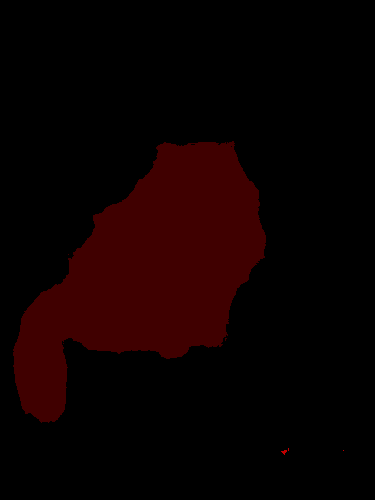}
&
    \includegraphics[width=0.14\textwidth]{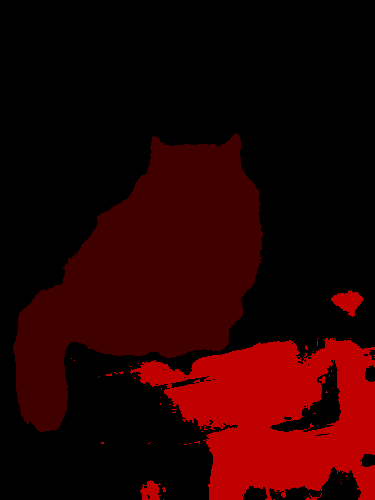}
\\
        {(a)} &{(b)} &{(c)} &{(d)} &\textbf{(e)} &{(f)} &{(g)}\\

\end{tabular}
}
\caption{Comparison of extracted CAMs and segmentation maps for our method and for ToCo \cite{ru2023token}. (a) input, (b) ground-truth mask, (c) CAM of ViT, (d) CAM of ToCo \cite{ru2023token}, (e) CAM of our method, (f) the masks obtained by ToCo, (g) the masks obtained by our method.}
\label{fig:cam_comparsion}
\end{figure}

\subsection{WSSS experiments}\label{sec:wsss_exps}

The WSSS pipeline presented in Sec.~\ref{sec:wsss} is less complex than recent methods~\cite{ru2023token}. It is evaluated on the PASCAL VOC 2012 dataset~\cite{everingham2010pascal}, which is a widely used benchmark in semantic segmentation. Following common practice, the VOC 2012 dataset is further augmented with the SBD dataset~\cite{hariharan2011semantic}. The augmented dataset consists of 10,582 training images, 1,449 validation images, and 1,456 test images. Training uses only the image-level labels. Evaluation is performed via the standard mIoU.

Evaluation is performed using a DeiT-small network~\cite{pmlr-v139-touvron21a} pretrained on ImageNet-1K. The decoder in the proposed architecture is a simple segmentation head, which comprises two $3\times 3$ convolutional layers with a dilation rate of 5, followed by a $1\times 1$ prediction layer. We set the auxiliary classifier to operate on top of the third to the last layer of the neural network. Training is performed with the AdamW optimizer~\cite{loshchilov2018decoupled}, with a learning rate that increases linearly to $6e^{-5}$ during the first 2500 iterations, and subsequently decays with a polynomial scheduler. The warm-up rate is set to $1e^{-6}$, and the decay rate is set to $0.9$. A batch size of 4 is used for a total of 80k iterations.

The results are presented in Tab.~\ref{tab_sem_seg}. 
Our method achieves a mIoU of 72.5 and 72.9 on the validation and test sets, respectively, compared to 71.1 and 72.2 achieved by the recent state-of-the-art single-stage method~\cite{ru2023token}. Additionally, we observe that our approach outperforms most of the multi-stage methods, with the exception of a method that employs additional supervision in the form of saliency maps (I + S). Note that ToCo, \cite{ru2023token} also report results with Deit-S; however, these results were lower than what is reported in the table. Tab.~\ref{table:classes_iou} lists the semantic segmentation results on the validation set for each class in the dataset.  Tab.~\ref{tab_pseudo_label_s} presents the results at the level of the pseudo labels produced on the training set, demonstrating the effectiveness of the solution earlier on in the pipeline.

Fig.~\ref{fig:cam_comparsion} presents typical samples of the Class Activation Maps (CAMs) and segmentation masks generated by our proposed method compared to those produced by ToCo~\cite{ru2023token}. This leading baseline uses complex pipelines to address the oversmoothing issue. Evidently, our method produces semantic segmentations with more precise object boundaries compared to this recent method. 

\textbf{Ablation Studies.\enspace} We conduct various experiments to examine the contribution of the different parts of our pipeline and to explore the effect of the centering loss for multiple configurations. These results are presented in Tab.~\ref{tab_ablation}. First, ViT with a single classification head at the last layer results in CAMs that provide a poor mIoU score of only 28.1 due to oversmoothing. Our correction term alone leads to a significant improvement in mIoU scores, to 56.1.

The auxiliary classification head of the intermediate layer does not help the CAM of the final layer when the centering term is not applied and helps slightly when it is added. We note that the CAM of the intermediate layer itself is much higher than the last layer counterpart, in the case of only $\mathcal{L}_{cls} + \mathcal{L}^{m}_{cls}$. Here, too, our correction term helps by a significant margin. 

It is evident that both $\mathcal{L}_{seg}$ and $\mathcal{L}_{reg}$ lead to an improvement in performance for the CAM of the last layer. The latter also improves the segmentation results (there is no segmentation head without $\mathcal{L}_{seg}$). Both the last layer CAM and the segmentation head benefit also from the centering term by a significant margin, with or without the regularization term. The CAM of the middle layer, however, shows a paradoxical behavior. While it benefits from the addition of the segmentation loss, both centering and the regularization term hurt its performance. It is likely that by adding the regularization term, the loss term associated with the intermediate head becomes much less prominent since the emphasis becomes the last layer (three loss terms are associated with the last layer). Similarly, the effect of $\mathcal L^m_{cls}$ among three or four loss term diminishes when adding the centering term, since this term allows for an effective optimization of the loss terms of the last layer (see the $\mathcal L_{cls}$ ablation).

\subsection{Experiments with GNNs}

Following previous work~\cite{guo2023contranorm}, evaluation is performed on two widely used citation graphs, Cora~\cite{mccallum2000automating} and CiteSeer~\cite{giles1998citeseer}, as well as two Wikipedia article networks, Chameleon and Squirrel ~\cite{rozemberczki2021multi}. 
We adopted the data split setting used in~\cite{kipf2016semi}, with train/validation/test splits of 60\%, 20\%, and 20\%, respectively.

In the Cora and CiteSeer graphs, nodes represent papers, and edges specify the citation relationships between papers. Nodes are classified according to their academic topics. The Cora dataset has a total of $2708$ nodes and $5278$ edges, the feature dimension of each node is $1433$ and each node is classified into 7 different classes. For the Citeseer dataset the graph contains $3327$ nodes with a total of $4552$ edges, each node is represented with a feature vector of dimension $3703$ and classified into one of 6 different classes.
We also used the Chameleon and Squirrel subsets of the Wikipedia network, where the nodes in each graph are Wikipedia page networks on specific topics, where nodes represent web pages and edges are the mutual links between them. Node features are the bag-of-words representation of informative nouns. The nodes are classified into 4 categories (for Squirrel) and 6 categories (for Chameleon) according to the number of the average monthly traffic of the page, the Squirrel dataset contains $5201$ nodes and $217073$ edges, where each node has a dimensionality of $2089$, for the Chameleon network it contains $2277$ nodes and $36101$ edges, and the dimensionality of the feature vector of each node is $500$.

We evaluate the performance of the proposed methods using the accuracy metric. Specifically, we report the classification accuracy achieved by the different methods for different numbers of layers. 
By examining the performance of the different methods across different number of layers, we can gain insights into their ability to prevent oversmoothing.

Following \cite{guo2023contranorm}, we utilized graph convolutional neural networks (GCNs) in our experiments. To ensure a fair comparison with existing methods, we kept the hidden dimension of each layer fixed at 32, and dropout rate at $0.6$.

Tab.~\ref{table:gnn-acc} lists the node classification accuracy statistics for the four datasets.
Except for the lower performance of our approach over Cora and CiteSeer for the most shallow models, the results demonstrate its efficacy in mitigating the issue of oversmoothing encountered by deeper GCN models, especially on these two datasets. In all cases, our method outperforms the recent works, including work designed to specifically overcome oversmoothing~\cite{zhao2019pairnorm,guo2023contranorm,ba2016layer}. 
We note that to avoid biasing the experiments, we use the same hyperparameters as ContraNorm \cite{guo2023contranorm}). Our experiments show that by changing these hyperparameters, we can greatly increase the performance gap in our favor, see the supplementary appendix. 

\section{Limitations}

Although our research provides a compelling mathematical explanation for the oversmoothing issue in certain Transformer architectures, we acknowledge that our analysis only considers relatively simple architectures without normalization layers and residual connections. To complement this analysis, we conducted an empirical simulation that illustrated the widespread presence of this problem in three popular variations of the Transformer architecture. Furthermore, our simulation demonstrated that our proposed correction effectively mitigates the oversmoothing phenomenon. Additionally, our theory primarily focuses on characterizing the problem of oversmoothing during the initialization phase. Although the behavior of the network at initialization significantly impacts its performance at the end of the training, it would be interesting to analyze the effect of the offset parameter on oversmoothing throughout the training process.

\section{Conclusions}

Much research is dedicated to evaluating different forms of layer ordering and normalization for Transformers~\cite{xiong2020layer,zhu2021iot,press2019improving,nguyen2019transformers}, perhaps since many of these models were developed before theoretical guidelines were provided. Similarly, in GNNs, the study of various variants aimed at training deeper networks is a very active field~\cite{zhao2019pairnorm,dasoulas2021lipschitz,chamberlain2021grand,xu2021optimization,li2021training}. As we demonstrate, in the case of Transformers, oversmoothing is an issue of similar magnitude for multiple forms of normalization, and all benefit from our correction term. For GNNs, we show that our simple term outperforms more elaborate normalization forms.

The centering attention layer we present is driven by the study of the singular values of an abstract attention model. 
Future studies can be dedicated to understanding the effect of the centering term on the inductive bias and priors of the resulting models through tools such as NNGP~\cite{hazan2015steps,lee2018deep} and NTK~\cite{jacot2018neural} kernels.

\bibliographystyle{plain}
\bibliography{main}

\newpage

\end{document}